\documentclass[conference]{IEEEtran}
\usepackage{microtype}
\usepackage{graphicx, caption}
\usepackage{subfigure}
\usepackage{booktabs}
\usepackage{hyperref}

\usepackage{xspace, float, booktabs}
\usepackage{xcolor}

\usepackage{algorithm}
\usepackage{algorithmic}

\newcommand{\model}{f}
\newcommand{\features}{N}
\newcommand{\score}{\varphi}

\usepackage{xspace}
\newcommand{\ourmethod}{O-Shap\xspace}

\usepackage{amsmath,amssymb,amsthm,amsfonts}
\usepackage{bm}

\newtheorem{proposition}{Proposition}

\newtheorem{definition}{Definition}

\begin{document}

\title{Owen-based Semantics and Hierarchy-Aware Explanation (\ourmethod) 
}

\author{
\IEEEauthorblockN{\textbf{Xiangyu Zhou}, \textbf{Chenhan Xiao}, \textbf{Yang Weng}}
\IEEEauthorblockA{\textit{Department of Electrical, Computer and Energy Engineering}}
\IEEEauthorblockA{Arizona State University, Tempe, Arizona, United States}
Email: \{xzhou185, cxiao20, yang.weng\}@asu.edu
}

\maketitle

%Shapley value-based methods have become foundational in explainable artificial intelligence (XAI), offering theoretically grounded feature attributions through cooperative game theory. However, their assumption of feature independence limits their effectiveness in real-world settings, particularly in images, where pixels often exhibit strong spatial or semantic dependencies. To address this, we propose \underline{O}wen-based \underline{S}ematics and \underline{H}ierarchy-\underline{A}ware Ex\underline{p}lanation (\ourmethod), a model-agnostic framework that constructs a semantics-aware hierarchical segmentation of input features and computes attributions using the Owen value, a hierarchical generalization of the Shapley value. This approach preserves the axiomatic guarantees of Shapley values. Unlike axis-aligned or SLIC-based segmentation, we design a segmentation that enforces semantic consistency across hierarchy levels by satisfying the $\mathcal{T}$-property, ensuring alignment between coarse and fine segments. The resulting hierarchy allows \ourmethod to prune redundant subset evaluations, improving both interpretability and scalability. Extensive experiments on image and tabular datasets demonstrate that \ourmethod consistently outperforms state-of-the-art baselines in attribution accuracy, semantic alignment, and computational efficiency.

\begin{abstract}
Shapley value-based methods have become foundational in explainable artificial intelligence (XAI), offering theoretically grounded feature attributions through cooperative game theory.
However, in practice, particularly in vision tasks, the assumption of feature independence breaks down, as features (i.e., pixels) often exhibit strong spatial and semantic dependencies.
To address this, modern SHAP implementations now include the Owen value, a hierarchical generalization of the Shapley value that supports group attributions.
While the Owen value preserves the foundations of Shapley values, its effectiveness critically depends on how feature groups are defined.
We show that commonly used segmentations (e.g., axis-aligned or SLIC) violate key consistency properties, and propose a new segmentation approach that satisfies the $\mathcal{T}$-property to ensure semantic alignment across hierarchy levels.
This hierarchy enables computational pruning while improving attribution accuracy and interpretability.
Experiments on image and tabular datasets demonstrate that O-Shap outperforms baseline SHAP variants in attribution precision, semantic coherence, and runtime efficiency, especially when structure matters.
\end{abstract}

% ###################  INTRODUCTION  ###################
\section{Introduction}

Interpreting the predictions of machine learning models is critical in domains where trust, transparency, and accountability are essential 
\cite{lundberg2017unified}.
% \cite{lundberg2017unified, adadi2018peeking, holzinger2017we, arrieta2020explainable, tjoa2020survey}. 
While methods like SHAP~\cite{lundberg2017unified} are widely adopted for their axiomatic foundation and model-agnostic design, their core assumption of feature independence often breaks down in structured or high-dimensional datasets. In vision tasks, for example, adjacent pixels frequently belong to coherent objects, and treating them independently can distort attribution. Specifically, as illustrated in Fig.~\ref{fig:bigpic}(a), pixels (features) A, B, and C are treated equally during the Shapley value calculation, ignoring the strong correlation between B and C, which belong to the same phone object. This independence assumption leads to attribution cancellation, where the marginal contributions of B and C partially cancel each other out despite their joint relevance. Such limitations also arise in other domains, such as sensor networks in IoT \cite{xiao2023distribution, Mostafa2025, xiao2024privacy, liao2022quickest, xiao2025attack} where spatial correlations exist, and in genomics \cite{herrero2001hierarchical} where genes co-express in modules. In vision tasks, this limitation is especially problematic since spatially adjacent pixels often form semantically meaningful structures (e.g., object parts) that contribute collectively to predictions. 

\begin{figure*}
    \centering
    \includegraphics[width = 1\linewidth]{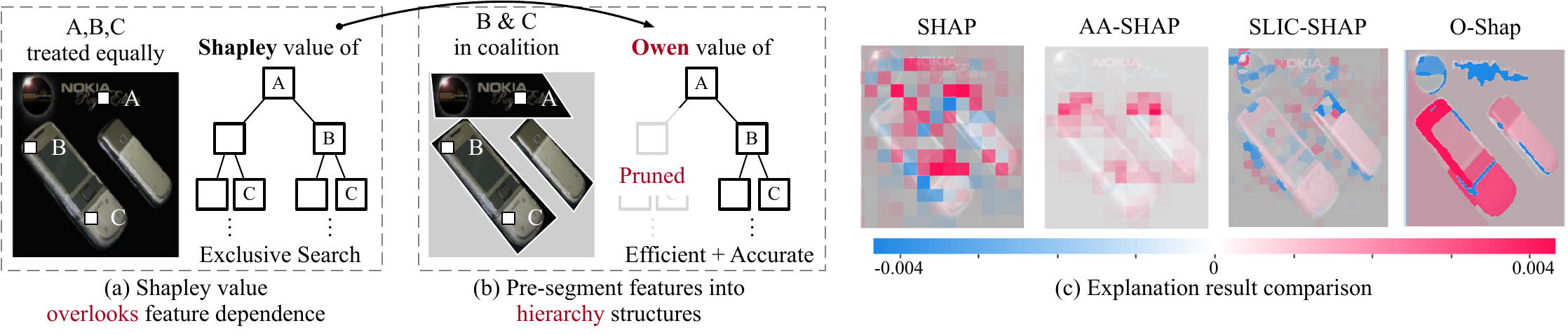}
    \caption{Overview of the motivation (a), design (b), and preliminary result (c)  of \ourmethod.}
    \label{fig:bigpic}
\end{figure*}

To address SHAP’s limitation in assuming feature independence, several extensions have been proposed. Causal SHAP\cite{heskes2020causal} incorporates causal models to compute Shapley values that respect conditional relationships among features, thereby enhancing interpretability and predictive accuracy\cite{frye2020asymmetric}. Kernel SHAP Extensions modify the sampling distribution to better reflect joint feature dependencies~\cite{aas2019explaining}. Graph-based methods, such as ShapG~\cite{zhao2024shapg}, construct feature graphs where edges encode correlations, enabling more structurally informed attributions. Tree-based approaches, such as those using conditional inference trees~\cite{redelmeier2020explaining}, impose hierarchical structures derived from learned splits to improve attribution consistency. While these methods offer valuable improvements, they often require predefined causal graphs, statistical assumptions, or domain-specific priors. Those resources are rarely available for image data. Moreover, these added complexities often come at the cost of violating the foundational axioms of Shapley values \cite{shapley1953value}, including {\it efficiency}, {\it linearity}, {\it symmetry}, and {\it dummy}.

For eliminating the reliance on extensive domain knowledge, we leverage hierarchical feature representation methods that pre-segment pixels into hierarchical structures, ensuring both informativeness and human-comprehensibility~\cite{gunning2017explainable}: prior studies suggest that users prefer explanations that strike a balance between simplicity and detail~\cite{harman1965inference,read1993explanatory,keil2006explanation}. While methods like Agglomerative Contextual Decomposition (ACD)\cite{singh2018hierarchical} effectively construct hierarchical structures in natural language processing (NLP), they depend on model-specific architectures and are not readily applicable to vision tasks. To address this gap, we propose a model-agnostic framework: {\it \underline{O}wen-based \underline{S}emantics and \underline{H}ierarchy-\underline{A}ware Ex\underline{p}lanation} ({\bf \ourmethod}), that dynamically constructs feature hierarchies via semantics-aware segmentation and computes attributions using the Owen value\cite{owen1977values}. As illustrated in Fig.\ref{fig:bigpic}(b), \ourmethod pre-organizes semantically related regions, such as the phone object and its background, to form a hierarchical structure that guides the Owen value computation. Our selection of the Owen value is motivated by its theoretical grounding: it preserves the foundational axioms of Shapley values~\cite{shapley1953value}, while allowing principled generalizations. Specifically, we reformulate the original {\it symmetry} and {\it dummy} axioms into {\it group symmetry} and {\it group dummy} \cite{owen1977values} based on the induced feature hierarchy, enabling faithful and structure-aware explanations in images.

To construct hierarchical structures over image pixels, existing segmentation strategies such as axis-aligned segmentation \cite{lundberg2017unified} (used in the official SHAP package) and Simple Linear Iterative Clustering (SLIC) \cite{6205760} (a standard superpixel method) often fail to capture semantic structure. Axis-aligned segmentation (AA-SHAP) partitions the image into uniform rectangular grids, while SLIC-based segmentation (SLIC-SHAP) clusters pixels by spatial and color proximity. Though simple and efficient, these methods are not inherently aligned with the semantic content of the image: they could fragment coherent objects or merge unrelated regions. As shown in Fig. \ref{fig:bigpic}(c), such misalignment results in noisy and imprecise attributions that fail to reflect meaningful object boundaries. In contrast, \ourmethod introduces a semantics-aware hierarchical segmentation framework designed to satisfy the $\mathcal{T}$-property of hierarchy \cite{li2022deep}, which enforces consistency between coarse- and fine-grained segments. Specifically, we generate the coarse layer using edge detection algorithm and iteratively refine finer layers through a graph-based merging algorithm that leverages attribution-aware edge weights. We prove that this procedure satisfies the $\mathcal{T}$-property, ensuring structural and semantic coherence across hierarchy levels. Empirically, \ourmethod outperforms both AA-SHAP and SLIC-SHAP, as demonstrated in Fig.~\ref{fig:bigpic}(c), where it cleanly separates screen, keyboard, and background elements.

Our main contributions are as follows:
\begin{itemize}
    \item We identify limitations of standard SHAP in vision tasks. To address this, we propose image-specific adaptations of Shapley axioms (e.g., a group-level dummy axiom) that better reflect the structure of visual data.
    \item We reveal a critical flaw in existing hierarchical SHAP methods: applying the Owen value without enforcing inter-group consistency can yield unstable or misleading attributions. To resolve this, we formalize the \emph{positive $\mathcal{T}$-Property} as a necessary condition for valid hierarchical segmentation, and design a bottom-up grouping algorithm that provably satisfies this property.
    \item We provide a formal complexity analysis showing that \ourmethod reduces the intractable cost of SHAP from $\mathcal{O}(2^{|N|})$ over all feature subsets to a polynomial-time complexity of $\mathcal{O}(|N|^2)$–$\mathcal{O}(|N|^3)$, where $|N|$ denotes the number of input features. 
\end{itemize}

O-Shap is evaluated on five image datasets and one tabular dataset, ensuring comprehensive validation across diverse data types. Extensive comparisons with seven SHAP variants across five metrics and execution time reveal that O-Shap provides superior explanations while significantly reducing computational costs. By integrating theoretical rigor through Owen-value derivation with practical efficiency via semantics-aware hierarchical segmentation, O-Shap represents a major advancement in XAI.

% ###################  Related Work  ###################
\section{Related Work}
\textbf{Challenges and Advancements in XAI Society.}\ 
Non-hierarchical explanation methods (i.e., methods that assign feature attributions without considering interdependencies or hierarchical structures), such as LIME~\cite{ribeiro2016should} and SHAP~\cite{lundberg2017unified}, have become widely used due to their simplicity and theoretical foundations. LIME relies on local linear approximations but lacks global consistency and robustness~\cite{zhang2019should, dieber2020model}. SHAP, leveraging cooperative game theory, improves upon LIME by ensuring local accuracy and consistency through Shapley values. However, its assumption of feature independence limits its ability to capture complex feature interactions, especially in high-dimensional datasets~\cite{covert2020understanding}. Various extensions, including Kernel SHAP~\cite{aas2019explaining} and Causal SHAP~\cite{heskes2020causal}, attempt to address these limitations by incorporating feature dependencies. Yet, these approaches often require predefined causal graphs or expensive joint distribution estimations, making them impractical for large-scale applications.

To overcome these challenges, hierarchical and structured explanation methods have emerged. Owen value-based techniques extend Shapley values by incorporating predefined hierarchies, allowing feature attributions to better capture structural dependencies~\cite{grabisch1999axiomatic}. Methods such as Asymmetric Shapley Values~\cite{frye2020asymmetric} and TreeSHAP~\cite{lundberg2020local} integrate causal and structural information to enhance interpretability. Additionally, graph-based methods~\cite{ying2019gnnexplainer} explicitly model feature relationships for more accurate explanations. However, these techniques often rely on domain-specific knowledge or predefined structures, restricting their adaptability across diverse datasets. In contrast, O-Shap dynamically constructs feature hierarchies based on semantic consistency, eliminating the need for prior knowledge while preserving the theoretical rigor of the Owen value. 

\textbf{Feature Pre-Segmentation into Hierarchical Structures.}
Hierarchical feature organization improves interpretability in image processing by grouping related pixels into higher-level structures, aligning with human perception. Superpixel segmentation techniques like SLIC~\cite{achanta2012slic} and Canny edge detection~\cite{canny1986computational} effectively partition images into meaningful regions, aiding object recognition~\cite{ren2003learning} and saliency detection~\cite{cheng2014global}. However, these methods lack direct integration with game-theoretic interpretability. Our O-Shap method bridges this gap by combining semantics-aware segmentation with the Owen value, enabling accurate, context-aware attributions. Unlike rigid axis-aligned segmentation, O-Shap dynamically adapts to image semantics, enhancing both precision and efficiency.

% ###################  Method  ###################
\section{Methods}
\label{sec:method}

\subsection{Notations and Shapley Value-Based XAI Method}
\label{sec:notation}

We denote the machine learning model under explanation as $\model(\cdot)$, and the input feature set by $\features = \{1, \cdots i \cdots, |\features|\}$, where $|\features|$ represents the total number of features. For vision tasks like image classification, the backbone of $\model(\cdot)$ is typically a convolutional neural network (CNN), due to its strong inductive biases for spatial locality, weight sharing, and translation invariance \cite{he2016deep}. In these tasks, each feature $i\in \features$ corresponds to a pixel $x_i$ in the input image $\bm{x}$.

Shapley value-based XAI methods assign an importance score to each feature $i$, quantifying its contribution by considering its marginal effect across all possible feature subsets:
\begin{equation}\label{eq:Shapleyvalue}
    \score_i = \sum_{S \subseteq \features\backslash \{i\}}
    \frac{1}{|\features|}\binom{|\features|-1}{|S|}
    [f_{S \cup \{i\}}(\bm{x}_{S \cup \{i\}}) - 
    f_{S}(\bm{x}_{S})],
\end{equation}
where $S \subseteq \features \setminus \{i\}$ ranges over all subsets of $\features$ excluding $i$, and the term $[f_{S \cup \{i\}}(\bm{x}_{S \cup \{i\}}) - 
    f_{S}(\bm{x}_{S})]$ 
represents the marginal contribution of feature $i$ when added to the subset $S$ \cite{shapley1953value}. The weighting factor $\frac{1}{|\features|} \binom{|\features|-1}{|S|}$ ensures a fair averaging over all subset sizes. In Eq. (\ref{eq:Shapleyvalue}), $\bm{x}_{S}$ denotes a perturbed version of the input image $\bm{x}$, where features (pixels) in subset $S$ are retained and the rest are replaced with baseline values, such as zeros or dataset means \cite{lundberg2017unified}.

The definition in Eq.~(\ref{eq:Shapleyvalue}) is not arbitrary. Rather, it emerges as the unique solution to a set of natural and desirable properties that any fair attribution method should satisfy \cite{shapley1953value}. These properties, also referred to as axioms, form the theoretical foundation of Shapley-based explanations. We formally state them below.
\begin{proposition}[Axiomatic Characterization of the Shapley Value]\label{pro:Shapley}
Let $\bm{x} \in \mathcal{X} = \{x_1, \cdots, x_i, \cdots, |\features|\}$ denote the input image,
$\features=\{1,\cdots,i,\cdots,|\features|\}$ denote the feature set, and let $f: \mathcal{X} \to \mathbb{R}$ be the model. The Shapley value $\score$ in Eq. (\ref{eq:Shapleyvalue}) is the unique attribution satisfying the following four axioms:
\begin{itemize}
\item \textbf{(Efficiency)}: The sum of feature attributions equals the total model output difference: $\sum_{i\in\features}\score_i = f(\bm{x}) - f(\bm{x}_{\varnothing})$.
\item \textbf{(Linearity)}: For any two models $\model_1, \model_2:\mathcal{X} \to \mathbb{R}$, the Shapley value of their sum equals the sum of their individual Shapley values:
\(
\score_i^{\model_1 + \model_2} = \score_i^{\model_1} + \score_i^{\model_2}, \  \forall i \in \features.
\)
\item \textbf{(Symmetry)}: If two features $i$ and $j$ contribute equally across all subsets (i.e., $\model(\bm{x}_{S \cup \{i\}}) = \model(\bm{x}_{S \cup \{j\}})$ for all $S \subseteq \features \setminus \{i, j\}$), they receive equal attribution:
\(
\score_i = \score_j.
\)
\item \textbf{(Dummy)}: If feature $i$ contributes nothing in any subset (i.e., $\model(\bm{x}_{S \cup \{i\}}) = \model(\bm{x}_{S})$ for all $S \subseteq \features \setminus \{i\}$), then it receives zero attribution:
\(
\score_i = 0.
\)
\end{itemize}
\end{proposition}

The axioms in Proposition \ref{pro:Shapley} ensure that the Shapley value in Eq.~(\ref{eq:Shapleyvalue}) fairly distributes the model output among input features by accounting for all possible interactions. However, we note that some of these axioms become ill-suited when applied to image data, motivating their reformulation.

\subsection{Reformulating Axioms for Image Data: Owen Value}
\label{sec:axioms}

While the Shapley value offers a theoretically sound attribution method under its standard axioms, two of these, \textit{Symmetry} and \textit{Dummy}, are fundamentally misaligned with image data and CNN-based models. In such settings, the model output depends not only on individual feature values but also on their spatial arrangement and interactions within local neighboring pixels. As we argue below, these characteristics violate the assumptions underlying the axioms \textit{Symmetry} and \textit{Dummy}, thereby motivating a principled reformulation of them tailored to image data.

In Proposition \ref{pro:Shapley}, the symmetry axiom assumes that if two features yield equal marginal contributions across all subsets, then they should receive equal attribution. However, in CNN-based models of vision tasks, the contribution of a pixel depends on its spatial location and its interaction with neighboring pixels through local receptive fields. Even if two pixels $i$ and $j$ are visually or statistically similar, they may activate different filters due to positional differences, resulting in asymmetric marginal effects. Thus, the global condition $\model(\bm{x}_{S \cup \{i\}}) = \model(\bm{x}_{S \cup \{j\}})$ for all $S \subseteq \features \setminus \{i, j\}$ rarely holds in image-based models, rendering the symmetry axiom inapplicable in practice. 

Similarly, the dummy axiom states that any feature $i$ with zero marginal contribution in all contexts must receive zero attribution, i.e., $\model(\bm{x}_{S \cup \{i\}}) = \model(\bm{x}_{S}), \forall S \subseteq \features \setminus \{i\} \Rightarrow \score_i=0$.
However, in CNN-based models, this statement can be violated due to the structure of convolutional kernels, which aggregate information across local receptive fields. Let $G\subseteq \features$ be a spatially coherent region (e.g., a superpixel) under the convolutional kernels, a pixel $i\in G$ may exhibit $\model(\bm{x}_{S \cup \{i\}}) = \model(\bm{x}_{S}), \forall S \subseteq G \setminus \{i\}$ because the output of the convolutional filter remains stable as long as the remaining pixels in $G$ are present. That is, pixel $i$'s influence is entirely absorbed by its neighbors within the receptive field. Thus, although no individual pixel in $G$ has marginal effect, the collective presence of the group is necessary to activate the corresponding convolutional response. The dummy axiom therefore incorrectly assigns $\score_i\approx0$ for all $i\in G$, failing to reflect the non-additive, group-level semantics encoded by convolutional kernels. This structural limitation motivates a reformulation of the dummy axiom to operate at the group level, consistent with the way CNNs pool and represent information.

To address the mismatch between the standard Shapley axioms and structured feature domains such as images, we reformulate the problematic \textit{Symmetry} and \textit{Dummy} axioms in Proposition \ref{pro:Shapley} to account for spatial grouping and localized dependencies. Let $\mathcal{G}=\{G_1,\cdots,G_k,\cdots,G_K\}$ denote a fixed hierarchical partition over the feature set $\features$, i.e., $\features = \bigcup_{k=1}^K G_k$ and $G_k\cap G_{k'} = \varnothing, k\neq k'$. Each element $G_k \in \mathcal{G}$ is a group of features (e.g., a superpixel or semantic segment). We retain the original forms of the \textit{Efficiency} and \textit{Linearity} axioms, but replace the remaining two as follows:
\begin{itemize}
\item \textbf{Group Symmetry.}
For any group $G_k\in \mathcal{G}$ and any two features $i,j\in G_k$, if $\model(\bm{x}_{S\cup\{i\}})=\model(\bm{x}_{S\cup\{j\}}), \forall S\subset G_k\setminus \{i,j\}$, then their attributions are equal: $\score_i = \score_j$. This axiom restricts symmetry to structurally coherent groups, avoiding the unrealistic assumption of global feature exchangeability.
\item \textbf{Group Dummy.}  
For any group \( G_k \in \mathcal{G} \), if
\(
f(\bm{x}_{S \cup G_k}) = f(\bm{x}_{S}), \quad \forall S \subseteq \mathcal{G} \setminus G_k,
\)
then every feature in \( G_k \) receives zero attribution: \( \score_i = 0 \) for all \( i \in G_k \). This axiom allows individual features to be redundant, as long as their collective presence has predictive utility.
\end{itemize}

These group-level axioms respect the compositional nature of structured inputs, where features often act in coordination rather than isolation. By enforcing symmetry and dummy axioms within semantically meaningful regions, rather than globally, they provide a more faithful alignment between theoretical assumptions and the inductive biases of CNN-based models. We now show that under these reformulated axioms, together with the original Efficiency and Linearity axioms, the unique attribution method is given by the Owen value \cite{owen1977values}.
\begin{proposition}[Uniqueness of the Owen Value under Group-Aware Axioms]\label{pro:Owen}
Let $\bm{x} \in \mathcal{X} = \{x_1, \cdots, x_i, \cdots, |\features|\}$ denote the input image,
$\features=\{1,\cdots,i,\cdots,|\features|\}$ denote the feature set, and let $f: \mathcal{X} \to \mathbb{R}$ be the model. Let $\mathcal{G}=\{G_1,\cdots,G_k,\cdots,G_K\}$ denote a fixed partition over the feature set $\features$ such that $\features = \bigcup_{k=1}^K G_k$ and $G_k\cap G_{k'} = \varnothing, k\neq k'$. Then, the unique attribution satisfying the four axioms: \textit{Efficiency}, \textit{Linearity}, \textit{Group Symmetry} and \textit{Group Dummy}, is \cite{owen1977values, owen1981modification}
\begin{align}
    \score_i^{\text{O}} =
    \sum_{\substack{T \subseteq \mathcal{G} \setminus \{G_k\}}}
    &\frac{1}{K}\binom{K-1}{|T|}
    \sum_{\substack{S \subseteq G_k \setminus \{i\}}}
    \frac{1}{|G_k|}\binom{|G_k|-1}{|S|}\cdot \nonumber \\
    &\left[\model\left(\bm{x}_{\bigcup T\cup S \cup \{i\}}\right) - \model\left(\bm{x}_{\bigcup T\cup S}\right)\right],
    \label{eq:Owenvalue-axiom}
\end{align}
for $i\in G_k\in \mathcal{G}$, where the subscript $\text{O}$ stands for Owen value.
\end{proposition}

This result is well established in cooperative game theory, and we refer readers to Owen~\cite{owen1977values} for the full proof.

\subsection{Owen Value under Multi-Layer Hierarchy}
\label{sec:owenvalue}

The Owen value in Eq. (\ref{eq:Owenvalue-axiom}) represents a unique attribution solution under a single-level hierarchical structure. We now extend this formulation to accommodate multi-layer hierarchies while preserving the same axiomatic foundations. This hierarchical generalization allows us to capture both fine-grained pixel-level details and coarse semantic groupings in images. 

Consider an image $\bm{x} = \{x_1, \cdots, x_{|\features|}\}$, whose features are organized hierarchically from coarse to fine granularity. At level 1, the image is segmented into coarse coalitions: 
\(
\mathcal{G}^1 = \{G^1_1, \cdots G^1_k \cdots, G^1_{K^1}\}
\)
where each $G^1_k$ is a coalition comprising multiple pixels, and 
\(
\bm{x} = \bigcup_{k=1}^{K^1} G^1_k
\)
,
\(G^1_k\cap G^1_{k'} = \varnothing\) for \( k\neq k'\). Each such coalition $G^1_k$ is further decomposed at level 2 as
\(
\mathcal{G}^2_k = \{G^2_1,\cdots, G^2_{K^2_k}\}
\)
and this recursive partitioning continues until level $L$, where each element corresponds to an individual pixel. To illustrate this, Fig.~\ref{fig:carseg} presents a 3-level ($L=3$) segmentation hierarchy of a car image. At level 1, the image is divided into two main coalitions: the car object ($G^1_1$) and the background ($G^1_2$). At level 2, the car object ($G^1_1$) is further decomposed into three semantically meaningful components ($G^2_1,G^2_2,G^2_3$), capturing distinct structural features. At level 3, the segmentation reaches pixel-level granularity.

\begin{figure}[H]
    \begin{center}
    \vskip -0.1in
    \includegraphics[width = 1\linewidth]{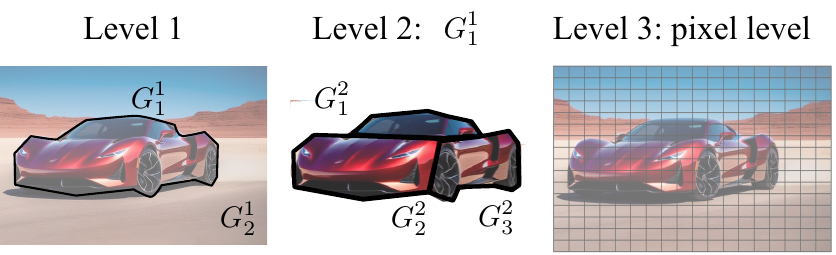}
    \end{center}
    \vskip -0.15in
    \caption{Example of the hierarchical structure of an image.}
    \vskip -0.1in
    \label{fig:carseg}
\end{figure}

In a $L$-level hierarchy structure, let pixel $x_i$ (at level $L$) be nested within
a sequence of enclosing coalitions:
$x_i \in G^{L-1} \subseteq \cdots \subseteq G^1$, where $G^l\in \mathcal{G}^l, l=1,\cdots,L-1$ is a coalition at level $l$. For simplicity, we omit coalition subscripts without loss of generality. The Owen value of the pixel $x_i$ is calculated as follows:
\begin{align}
    \score_i^{\text{O}} =
    &\sum_{\substack{S_1 \subseteq \mathcal{G}^1 \setminus \{G^1\}}} \cdots 
    \sum_{\substack{S_{L-1} \subseteq \mathcal{G}^{L-1} \setminus \{G^{L-1}\}}}
    \sum_{\substack{S_L \subseteq G^{L-1} \setminus \{x_i\}}}\nonumber \\
    &
    \prod_{l=1}^{L-1} \frac{1}{|\mathcal{G}^l|} 
    \binom{|\mathcal{G}^l| - 1}{|S_l|}
    \cdot \frac{1}{|G^{L-1}|} 
    \binom{|G^{L-1}| - 1}{|S_L|} \cdot\nonumber \\
    &
    \left[
    f\left( \bigcup_{l=1}^{L} S_l \cup \{x_i\} \right) - 
    f\left( \bigcup_{l=1}^{L} S_l \right)
    \right].
    \label{eq:Owenvalue-multi-level}
\end{align}

Compared to Eq. (\ref{eq:Owenvalue-axiom}), this generalized formulation in Eq. (\ref{eq:Owenvalue-multi-level}) extends the Owen value calculation to multi-level hierarchical structure, ensuring that feature contributions are accurately represented from higher-level coalitions down to individual pixels. Beyond preserving the axiomatic guarantees in Proposition~\ref{pro:Owen}, the hierarchical formulation of Eq. (\ref{eq:Owenvalue-multi-level}) provides substantial computational advantages over standard Shapley-based methods. In particular, SHAP suffers from exponential complexity due to its exhaustive evaluation over all $2^{|N|-1}$ possible feature subsets, where $|N|$ is the number of features. This makes the exact SHAP calculation impractical for high-dimensional inputs such as images. In contrast, Eq. (\ref{eq:Owenvalue-multi-level}) leverages the hierarchy structure to prune the search space, avoiding redundant evaluations. The resulting complexity is analyzed in Proposition \ref{pro:owen-complexity}.

\begin{proposition}[Computational Efficiency] 
\label{pro:owen-complexity}
Let $|N|$ denote the number of input features. Computing the Shapley value (Eq. (\ref{eq:Shapleyvalue})) requires an exponential complexity of $\mathcal{O}(2^{|N|})$. In the hierarchical Owen value (Eq.~(\ref{eq:Owenvalue-multi-level})), the complexity reduces to $\mathcal{O}(2^{\sum_{l=1}^L |\mathcal{G}^l|})$. Specifically, for a balanced partition where $|\mathcal{G}^l|\equiv n, l=1,\cdots,L$, the total computational cost simplifies to $\mathcal{O}(|N|^{n\cdot\log_{n}{2}})$, which is polynomial in $|N|$.
\end{proposition}

\begin{proof}
Let $|\mathcal{G}^l|$ denote the number of coalitions at level $l$, then the total number of subset evaluations of Eq.~(\ref{eq:Owenvalue-multi-level}) is 
\(
\mathcal{O}\left( \prod_{l=1}^{L} 2^{|\mathcal{G}^l|} \right) = \mathcal{O}\left( 2^{\sum_{l=1}^{L} |\mathcal{G}^l|} \right).
\)
Assuming a balanced hierarchy where each level contains $|\mathcal{G}^l| \equiv |N|^{1/L} = n$ coalitions, we have:
\begin{align}
    \mathcal{O}\left( 2^{\sum_{l=1}^{L} |\mathcal{G}^l|} \right) &= 
    \mathcal{O}\left( 2^{L\cdot n} \right) = \mathcal{O}\left( 2^{\log_n|N|\cdot n} \right)\nonumber \\
    &= \mathcal{O}\left( 2^{\frac{\log_2|N|}{\log_2n}\cdot n} \right)
    = \mathcal{O}\left(|N|^{n\cdot\log_{n}{2}}\right) \label{eq:owen-complexity}. \qedhere
\end{align}
\end{proof}

To illustrate the computational benefit, consider computing the exact Shapley value for $|N|=50$ features. This requires  $2^{50} \approx 1.12e^{16}$ times of evaluation. In contrast, a 3-level hierarchy with partition sizes $2\times5\times5=50$ reduces the computation to $2^{2+5+5}=2^{12}=4096$ times of evaluation.
Moreover, for balanced partitions where each level contains $n$ elements, the exponent term $n\cdot\log_n{2}$  in Proposition \ref{pro:owen-complexity} increases slowly from $2$ to $3$ as $n$ ranges from $2$ to $10$. As a result, the overall complexity of \ourmethod remains polynomial in practice, typically falling between $\mathcal{O}(|N|^2)$ and $\mathcal{O}(|N|^3)$.

\subsection{A Semantics-Aware Segmentation}
\label{sec:segmentation}

The Owen value in Eq. (\ref{eq:Owenvalue-multi-level}) critically depends on a well-defined hierarchical structure $\mathcal{G}_1,\cdots,\mathcal{G}_L$ that faithfully captures inter-feature relationships. However, existing segmentation strategies commonly used in XAI, such as axis-aligned partitioning \cite{lundberg2017unified} (as implemented in the official SHAP package) and Simple Linear Iterative Clustering (SLIC)~\cite{6205760}, fall short in this regard. Axis-aligned segmentation imposes a rigid grid over the image, disregarding content boundaries, while SLIC forms superpixels based on low-level features like color and spatial proximity. Though computationally efficient, both approaches lack semantic understanding and often result in either the fragmentation of coherent objects or the merging of semantically unrelated regions.

To ensure that the hierarchical segmentation used in Eq.(\ref{eq:Owenvalue-multi-level}) faithfully reflects the underlying semantics of an image, we adopt the following established definitions of \textit{hierarchical semantics consistency} \cite{bi2011multi, li2022deep} in image applications.

\begin{definition}[Positive $\mathcal{T}$-Property]
\label{def:T-property}
Let $G_i \in \mathcal{G}_i$ and $G_j \in \mathcal{G}_j$ be segments from different hierarchical levels such that $G_i \subseteq G_j$. If the model $\model(\cdot)$ assigns a positive label to $G_i$ for a given category, then $G_j$ must also be assigned a positive label for the same category.
\end{definition}

% \begin{definition}[Positive $\mathcal{T}$-Constraint]
% \label{def:T-property-euiv}
% For segments $S_i \in \mathcal{G}_i$ and $S_j \in \mathcal{G}_j$ with $S_i \subseteq S_j$, it holds that $\model(S_j) \geq \model(S_i)$ if $S_i$ is labeled positive.
% \end{definition}

% Definition \ref{def:T-property-euiv} provides a sufficient condition~\cite{li2022deep} that can be evaluated mathematically, serving as a theoretical criterion for segmentation validity in hierarchical feature attribution. 

Definition~\ref{def:T-property} \cite{li2022deep} (also known as positive $\mathcal{T}$-property) enforces semantic consistency by ensuring that if a child segment is recognized by the model as belonging to a particular category, e.g., identified as a dog because it contains a dog’s face, then its parent segment, which encompasses the same content, must not contradict this decision. However, not all segmentation methods used in XAI satisfy these semantic requirements. For example, axis-aligned segmentation, which divides the image into uniform rectangular grids, fails to guarantee semantically meaningful structures. Such partitions are often misaligned with object boundaries, leading to inconsistent labels across the hierarchy. This shortcoming is formally stated as follows.

\begin{proposition}\label{cor:grid_violation}
Axis-aligned segmentation does not guarantee satisfaction of the positive $\mathcal{T}$-Property.
\end{proposition}

\begin{proof}
Let $G_i \subseteq G_j$ be two nested segments generated by an axis-aligned partitioning scheme. Suppose $G_i$ captures a semantically coherent feature (e.g., a dog's face) and is assigned a positive label by the model: $\model(G_i) = 1$. Since axis-aligned segmentation divides the image without regard to semantic boundaries, $G_j$ may include irrelevant or contradictory regions (e.g., background clutter) that dilute the semantic signal. As a result, the model's aggregated prediction for $G_j$ may be negative: $\model(G_j) = 0$. This violates the positive $\mathcal{T}$-Property.
\end{proof}

To overcome this limitation, we propose a hierarchical segmentation framework explicitly designed to satisfy the positive $\mathcal{T}$-Property. Our approach consists of two steps: initial edge-based segmentation followed by semantic-aware graph-based merging.
\begin{itemize}
    \item Initial Segmentation. The bottom-level hierarchy $\mathcal{G}^L$ is generated using Canny edge detection~\cite{rong2014improved}. Each segment $G_i \in \mathcal{G}^L$ corresponds to a connected component enclosed by detected edges, ensuring that initial segments align with strong local image gradients.
    \item Hierarchical Merging. Higher-level segmentations $\mathcal{G}^{L-1}, \dots, \mathcal{G}^1$ are generated by iteratively merging segments based on attribution similarity. At each level $l$, we construct a graph $(V_l, E_l)$, where nodes correspond to segments in $\mathcal{G}^l$ and edges connect spatially adjacent segments. The edge weight between spatially adjacent segments $G_i$ and $G_j$ is defined as:
    \(
    w(e_{ij}) = |\model(G_i) - \model(G_j)|
    \)
    where $\model(G_i)$ denotes the model attribution score when only pixels in $G_i$ are retained and others are masked by a mean value. The merging process seeks a segmentation $\mathcal{G}^{l-1}$ that minimizes inter-segment semantic disparity:
    \begin{equation}\label{eq:merge}
    \min_{\mathcal{G}^{l-1}} \sum_{e_{ij} \in E_l} w(e_{ij}) \cdot \mathbb{I}[G_i \text{ and } G_j \text{ are merged}],
    \end{equation}
    where the indicator function $\mathbb{I}$ encodes the merging decisions. This merging proceeds until the entire image is represented as a single segment at the topmost level. By merging only segments with similar semantic attributions, the resulting hierarchy maintains semantic coherence and enforces the $\mathcal{T}$-property, as shown in Proposition
    \ref{cor:hierarchical_satisfaction}.
\end{itemize}

\begin{proposition}[Satisfaction of the $\mathcal{T}$-Property]
\label{cor:hierarchical_satisfaction}
The hierarchical segmentation constructed by Eq.~(\ref{eq:merge}) with $w(e_{ij}) = |\model(G_i) - \model(G_j)|$ satisfies the positive $\mathcal{T}$-Property.
\end{proposition}

\begin{proof}
Let $G_i \in \mathcal{G}^l$ and $G_j \in \mathcal{G}^{l-1}$ such that $G_i \subseteq G_j$, i.e., $G_j$ is formed by merging a subset of segments at level $l$ that includes $G_i$. Assume $G_i$ is labeled positive, meaning $\model(G_i) \geq \tau$ for some positive threshold $\tau$. From the merging criterion in Eq.~(\ref{eq:merge}), segments are merged only if their attribution scores are similar, i.e., $|\model(G_k) - \model(G_{k'})|$ is small for any pair merged into $G_j$. Let $\{G_i, G_{i_2}, \dots, G_{i_m}\} \subseteq \mathcal{G}^l$ be the collection of segments merged into $G_j$. Then by construction, 
\(
|\model(G_i) - \model(G_{i_t})| \leq \epsilon, \quad \forall t=2,\dots,m
\)
for some small $\epsilon > 0$ determined by the stopping criterion of the merge process. Since $\model(G_i) \geq \tau$, it follows that each $\model(G_{i_t}) \geq \tau - \epsilon$. Hence, the attribution score of the merged segment $G_j$ is approximately the average of semantically similar positives:
\(
\model(G_j) = \frac{1}{m}\sum_{t=1}^m \model(G_{i_t}) \geq \tau - \epsilon.
\)
If $\epsilon$ is sufficiently small (which is ensured by the design of the merging threshold), then $\model(G_j) \geq \tau' > 0$ for some \(\tau'\), implying $G_j$ is also labeled positive.
\end{proof}

Given an input image $\bm{x}$, the overall process for \ourmethod is summarized into Algorithm \ref{alg:hierarchical-owen}.

\begin{algorithm}[h]
\caption{Owen-based Semantics and Hierarchy-Aware Explanation (\ourmethod)}
\label{alg:hierarchical-owen}
\textbf{Input:} model $f(\cdot)$, an RGB Image $\bm{x}$, edge detection thresholds $T_{\text{lower}}$, $T_{\text{upper}}$
\newline
\textbf{Output:} Owen value for $x_i\in\bm{x}$
\begin{algorithmic}[1]
\STATE {\bf Step 1: initial segmentation via Canny edge detection}
\STATE Apply Gaussian smoothing to $\bm{x}$
\STATE Compute gradient magnitudes and directions. Apply non-maximum suppression to thin edges
\STATE Apply double thresholding with $T_{\text{lower}}$ and $T_{\text{upper}}$. Track edge by hysteresis to obtain edge-defined segments $\mathcal{G}$
\STATE {\bf Step 2: semantics-aware graph-based merging}
\STATE Initialize hierarchical structure $\mathcal{G}_L = \mathcal{G}$
\FOR{$l = L$ to $1$}
\STATE Construct a weighted graph $(V_l, E_l)$, where $V_l$ represents segments from $\mathcal{G}_l$, and edges $E_l$ connect spatially adjacent segments.
\FOR{each edge $e_{ij} \in E_l$}
\STATE Compute edge weight $ w(e_{ij}) = |\model(G_i) - \model(G_j)|$ based on model score differences
\ENDFOR
\STATE Merge segments in $\mathcal{G}_l$ using Eq. (\ref{eq:merge}) to form $\mathcal{G}_{l-1}$
\ENDFOR
\STATE {\bf Step 3: Owen value computation}
\FOR{each pixel $x_i$ at the pixel level $L$}
        \STATE Update the score $\score^{\text{O}}_i$ using Eq. (\ref{eq:Owenvalue-multi-level}).
\ENDFOR
\STATE \textbf{Return} Owen values $\score^{\text{O}}_i$ for all pixels
\end{algorithmic}
\end{algorithm}

\section{Experiments}
\label{sec:experiment}

We empirically evaluate \ourmethod across multiple datasets and baselines to validate the following: (1) the advantage of using Owen value over standard SHAP for hierarchical explanations, (2) the limitation of Owen value under improper segmentation, that does not satisfy the positive $\mathcal{T}$-property, (3) the computation efficiency of \ourmethod.

\textbf{Datasets.}\  
\ourmethod is evaluated on five publicly available image datasets and one tabular dataset to consider non-image data. 
For one-class image classification tasks, we use the following datasets. 
(1) {\bf MRI dataset}: The Brain Tumor MRI dataset~\cite{msoud_nickparvar_2021} with four categories represents critical applications in medical imaging. 
(2) {\bf Tiny-ImageNet dataset}: A subset of ImageNet with 200 classes~\cite{ILSVRC15}, offering a challenging benchmark for diverse life-like imagery. 
(3) {\bf ImageNet-S50 dataset}: A curated subset of ImageNet~\cite{gao2022luss} with 50 categories. The inclusion of benchmark object masks in this dataset makes it particularly suitable for evaluating XAI methods. 
To consider multi-class image classification tasks where images belong to several categories simultaneously, we extend our evaluation to the following. 
(4) {\bf CelebA dataset}: A human facial image dataset~\cite{liu2015faceattributes} with multiple attributes. 
(5) {\bf PASCAL-VOC-2012 dataset}: A widely-used benchmark~\cite{pascal-voc-2007} for multi-object detection and segmentation of life-like imagery. 
For non-image datasets, we utilize (6) \textbf{Adult Census Income dataset}~\cite{adult_2}, which consists of demographic and employment information for 48,842 individuals across 14 features.

\textbf{Baselines.}\  
We compare \ourmethod against a diverse set of baseline methods, encompassing the traditional SHAP framework and its recent variants designed to improve interpretability and computational efficiency.
(1) {\bf SHAP} \cite{lundberg2017unified}, the foundational method for Shapley value-based explanations, assumes feature independence in its attributions. 
(2) {\bf AA-SHAP} \cite{lundberg2017unified}, a variant of SHAP, incorporates predefined axis-aligned (AA) segmentation as implemented in the SHAP official library. 
(3) {\bf Gradient SHAP} \cite{lundberg2017unified} and (4) {\bf Integrated Gradients} \cite{sundararajan2017axiomatic} extend Shapley values by integrating gradient information, providing enhanced interpretability and computational efficiency. 
(5) {\bf Occlusion} \cite{zeiler2014visualizing}, a perturbation-based method, evaluates feature importance by systematically occluding input regions. 
(6) {\bf RISE} \cite{Petsiuk2018rise}, a randomized approach, generates feature attributions through masked input sampling, offering flexibility across diverse models. 
(7) {\bf h-SHAP} \cite{teneggi2022fast}, a hierarchical extension of SHAP that leverages structured attributions for improved consistency in hierarchical contexts. 

\textbf{Evaluation metrics.}\  
To assess the quality of the explanations, we employ four widely used metrics. (1) Energy-Based Pointing Game (EBPG) \cite{wang2020score} evaluates the alignment between explanations and ground-truth energy distributions. (2) Mean Intersection over Union (mIoU) assesses the overlap between predicted and ground-truth regions. (3) Bounding Box (Bbox) \cite{schulz2020restricting} measures the alignment of explanations with annotated bounding boxes. (4) Area Over the Perturbation Curve (AOPC) \cite{samek2016evaluating} measures the robustness of explanations by analyzing prediction changes under feature perturbations.

\textbf{Implementation details}.\ 
For each dataset, we fine-tune a ResNet model to achieve an overall classification accuracy exceeding 85\%, serving as the target model for explanation. Prior to explanation, all input images are resized to $224 \times 224$ and standardized using mean and standard deviation normalization. We then apply our semantics-aware hierarchical segmentation method. Specifically, the lower and upper thresholds for initial segmentation, $T_{\text{lower}}$ and $T_{\text{upper}}$, are set to the 75th and 90th percentiles, respectively. The hierarchical depth $L$ is determined adaptively, typically converging within 4 to 5 levels when all meaningful feature regions are fully merged. All experiments were conducted on a single NVIDIA RTX-4080-16GB.

\subsection{Evaluating the Limitations of Shapley-Based Attribution}

We begin by demonstrating the limitations of Shapley value-based methods in image data with high inter-feature correlations, using the Brain Tumor MRI dataset~\cite{msoud_nickparvar_2021}. In MRI images, tumor regions exhibit dark color patterns that contrast with surrounding normal brain tissue, leading to strong intra-region pixel correlations. However, SHAP’s assumption of feature independence fails to capture these dependencies, resulting in suboptimal explanations where attribution heatmaps are dispersed across the brain rather than focused on the tumor (see left panel of Fig. \ref{fig:SHAP-fail-tumor}). This reflects SHAP’s inability to model structured feature correlations, reducing interpretability in medical imaging and other domains with spatially or semantically dependent features.

\begin{figure}[H]
    \centering
    \vskip -0.1in
    \includegraphics[width = 1\linewidth]{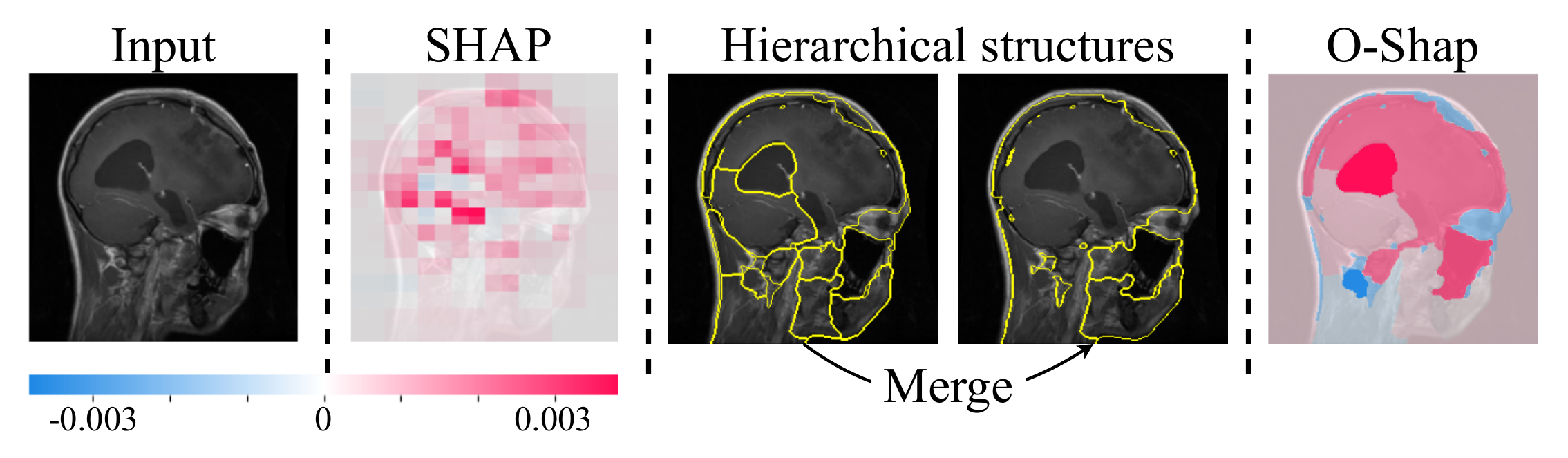}
    \vskip -0.1in
    \caption{Limitation of SHAP and a hierarchy-aware solution.}
    \vskip -0.1in
    \label{fig:SHAP-fail-tumor}
\end{figure}

To address this, our proposed \ourmethod constructs a multi-level feature structure on the tumor image, enabling correlated regions such as tumors to be treated as coalitions rather than independent pixels during Owen value computation. Fig. \ref{fig:SHAP-fail-tumor} (middle) visualizes the hierarchical layers, demonstrating how correlated regions are effectively grouped. As validated by Proposition \ref{cor:hierarchical_satisfaction}, \ourmethod satisfies hierarchical consistency, ensuring that the segmentation aligns with semantic structure. The resulting heatmaps (see right panel of Fig. \ref{fig:SHAP-fail-tumor}) show that O-Shap produces focused, interpretable explanations centered on the tumor, in contrast to SHAP’s scattered outputs. These results highlight the practical advantages of \ourmethod over SHAP in domains where feature correlations are critical.

\subsection{Evaluating the Role of Segmentation in \ourmethod}

To validate the importance of our proposed semantics-aware segmentation method in Section~\ref{sec:segmentation}, we conduct an ablation study comparing four variants: (1) SHAP (basic SHAP without segmentation), (2) AA-SHAP (with axis-aligned segmentation), (3) SLIC-SHAP (with SLIC segmentation), and (4) \ourmethod (with our proposed semantics-aware segmentation). Fig.~\ref{fig:segment_comp} presents qualitative results. The top three rows correspond to relatively simple images with clear object boundaries and minimal background noise, while the bottom three rows depict more complex and realistic scenarios. Among all variants, \ourmethod consistently produces the most accurate and interpretable heatmaps by localizing attributions to semantically meaningful regions while suppressing irrelevant artifacts. In contrast, AA-SHAP suffers from inflexible grid boundaries that poorly align with natural object contours, leading to coarse and misaligned attributions. SLIC-SHAP, though more adaptive, relies solely on low-level visual cues such as color and proximity, which can group semantically unrelated pixels and fragment cohesive features. These deficiencies result in scattered or diluted attributions, especially in high-correlation regions. Notably, both AA-SHAP and SLIC-SHAP fail to satisfy the $\mathcal{T}$-Property (Definition~\ref{def:T-property}), limiting their hierarchical consistency and undermining their interpretability. These findings underscore the critical role of segmentation in hierarchical attribution: incorporating semantics-aware, hierarchy-preserving segmentation significantly enhances explanatory coherence.

\begin{figure}[h]
    \begin{center}
    \vskip -0.1in
    \includegraphics[width = 0.95\linewidth]{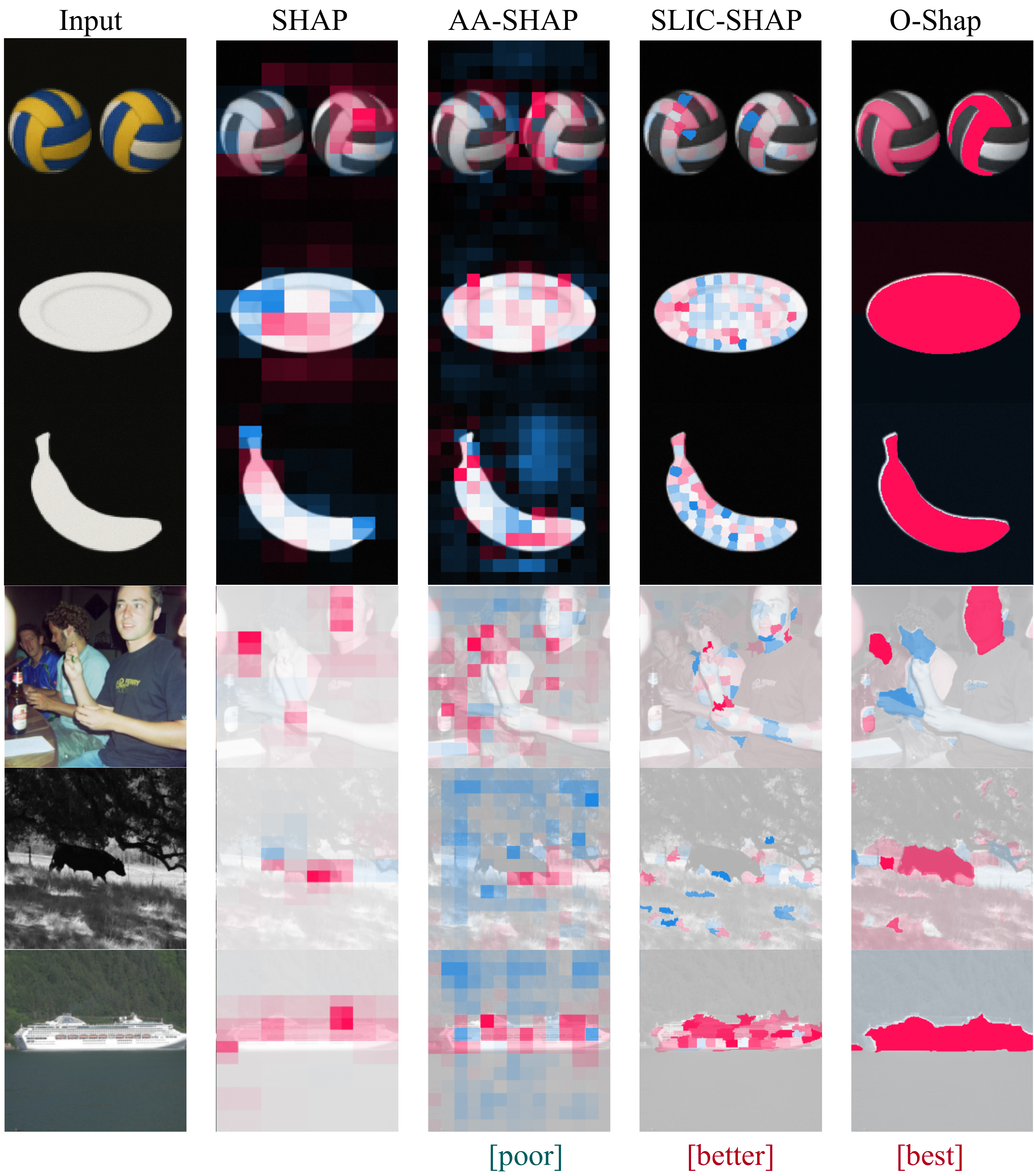}
    \end{center}
    \vskip -0.1in
    \caption{Explanation heatmaps under various segmentations.}
    \vskip -0.1in
    \label{fig:segment_comp}
\end{figure}

\subsection{Comprehension Comparison among Many Baselines on Various Image Dataset}

To comprehensively evaluate \ourmethod, we compare it against state-of-the-art baselines using ResNet50 on the ImageNet-S50 dataset. Fig. \ref{fig:compare-ImageNet-S50} shows explanation heatmaps, with rows representing image categories and columns corresponding to different methods. \ourmethod achieves the highest mIoU score, demonstrating its superior focus on target objects. In contrast, SHAP assumes feature independence, leading to scattered attributions. GradSHAP and Integrated Gradients improve efficiency but are influenced by gradient noise, while Occlusion captures object regions robustly but suffers from coarse masks. RISE provides diverse but imprecise explanations due to randomized sampling. h-SHAP improves SHAP through hierarchical structures but relies on fixed segmentations, limiting localization accuracy. By leveraging hierarchy-consistent segmentation and Owen value-based attributions, \ourmethod produces precise, context-aware explanations, reflected in its leading mIoU score.

\begin{figure}[ht]
    \centering
    \vskip -0.1in
    \includegraphics[width = 1\linewidth]{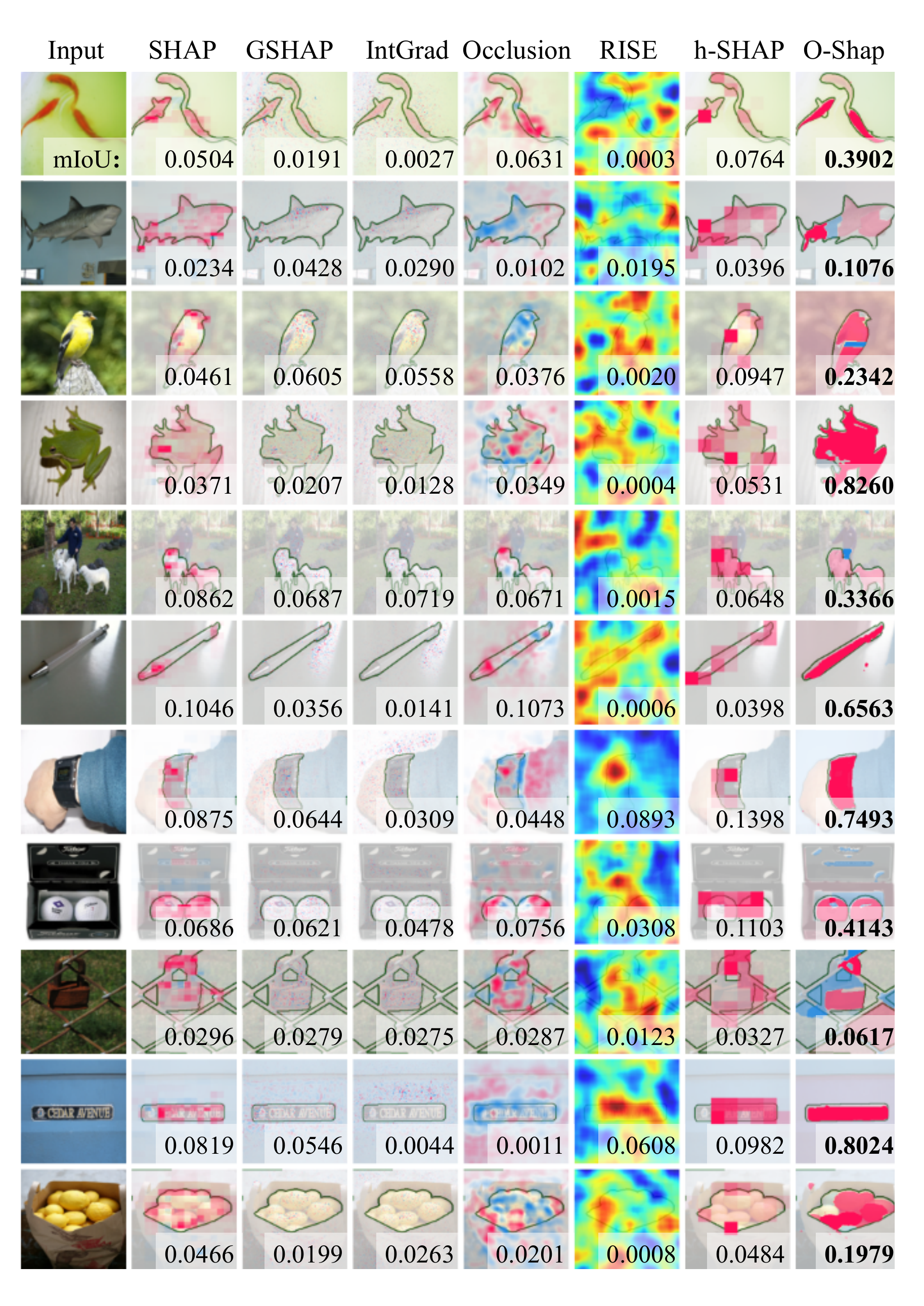}
    \vskip -0.1in
    \caption{Explanations in ImageNet-S50 with mIoU scores.}
    \vskip -0.15in
    \label{fig:compare-ImageNet-S50}
\end{figure}

To provide a more comprehensive quantitative comparison, Table \ref{metrics} compares the performance of \ourmethod with six baseline methods across five metrics for a ResNet50 classifier trained on the ImageNet-S50 dataset. \ourmethod achieves the highest mIoU score (0.3375), demonstrating its superior ability to generate precise and localized explanations that align closely with object boundaries. While SHAP achieves the second-best mIoU (0.2780), its assumption of feature independence limits its capacity to focus on high-correlation features effectively. For the EBPG metric, \ourmethod outperforms all methods with a score of 0.5692, leveraging its hierarchical structure to better align with the ground-truth energy distributions, while SHAP follows closely with 0.5659. Notably, SHAP achieves the highest Bbox score (0.6149), reflecting its strength in capturing bounding box-level explanations, with h-SHAP trailing as the second-best at 0.4937. However, \ourmethod lags in Bbox performance with a score of 0.4707, suggesting that there is room for improvement in aligning with bounding-box annotations. For F1 and AUC, \ourmethod shows the best overall performance (0.4172 and 0.5867, respectively). These results highlight the strength of \ourmethod in segmentation-based metrics such as mIoU while maintaining competitive performance in other metrics, showcasing its robustness and effectiveness in generating high-quality explanations.

\begin{table}[H]
\caption{Evaluation of XAI methods on ImageNet-S50.}
\vskip -0.05in
\label{metrics}
\begin{center}
\resizebox{\columnwidth}{!}{%
\begin{tabular}{cccccccc}
\toprule
\bf Metric & SHAP & GradSHAP & IntGrad & Occlusion & RISE & h-SHAP & \ourmethod\\ 
\toprule
\toprule
\bf EBPG & 0.5659 & 0.5044 & 0.5026 & 0.4700 & 0.2653 & 0.4779 & \bf{0.5692}\\
\bf mIoU & 0.2780 & 0.1650 & 0.1422 & 0.1794 & 0.1304 & 0.2248 & \textbf{0.3375}\\
\bf Bbox & \textbf{0.6149} & 0.3904 & 0.3646 & 0.4015 & 0.3729 & 0.4937 & 0.4707\\
\bf F1 & 0.4017 & 0.3076 & 0.3077 & 0.2904 & 0.3815 & 0.4109 & \bf{0.4172}\\
\bf AUC & 0.5588 & 0.5027 & 0.5027 & 0.4905 & 0.5000 & 0.5734 & \bf{0.5867}\\
\bottomrule
\end{tabular}
}
\end{center}

\end{table}

To evaluate \ourmethod's performance on multi-label classification tasks, we use the CelebA dataset \cite{liu2015faceattributes} and train a sequential CNN model. We focus on images that belong to two categories simultaneously: $Wearing\_Hat$ and $Eyeglasses$. These categories are challenging as they correspond to distinct regions of the human face, which requires the explanation method to accurately localize the features associated with each category. \ourmethod effectively addresses this challenge by generating heatmaps that concentrate on the respective areas: focusing on the hair region for the $Wearing\_Hat$ category and the eye region for the $Eyeglasses$ category. This precise localization highlights \ourmethod's ability to handle multi-label classification scenarios by leveraging its hierarchy-consistent segmentation and Owen value-based attributions. In contrast, baseline methods tend to produce scattered or overlapping attributions, failing to distinctly capture the features corresponding to each category. These results underscore \ourmethod's advantage in providing semantically meaningful and category-specific explanations for complex multi-label tasks.

\begin{figure}[ht]
    \begin{center}
    \includegraphics[width = 1\linewidth]{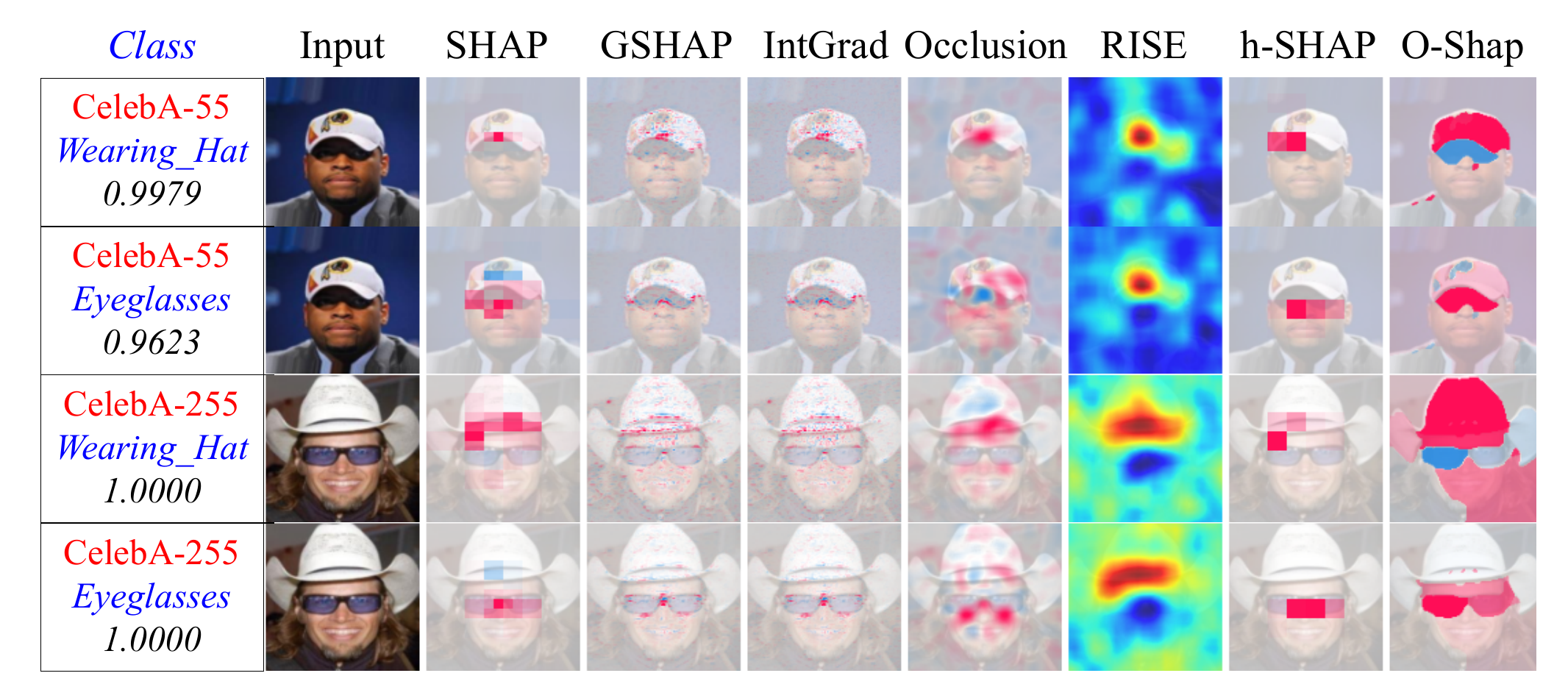}
    \end{center}
    \vskip -0.1in
    \caption{Comparison on a multi-label classification task on CelebA, classified by the ResNet50 model.}
    \vskip -0.1in
    \label{fig:CelebA}
\end{figure}

% \begin{figure}[H]
%     \begin{center}
%     \includegraphics[width = 1\linewidth]{Figures/SHAP.png}
%     \includegraphics[width = 1\linewidth]{Figures/IG.png}
%     \includegraphics[width = 1\linewidth]{Figures/GSHAP.png}
%     \end{center}
%     \caption{SHAP, Integrated Gradients and Gradient SHAP}
%     \label{fig:shap}
% \end{figure}

% \begin{figure}[H]
%     \begin{center}
%     \includegraphics[width = 1\linewidth]{Figures/Occlusion.png}
%     \includegraphics[width = 1\linewidth]{Figures/RISE.png}
%     \end{center}
%     \caption{Occlusion and RISE}
%     \label{fig:shap}
% \end{figure}

% \clearpage
% \begin{figure}[H]
%     \begin{center}
%     \includegraphics[width = 1\linewidth]{Figures/ResNet50onTinyImageNet_comparison.png}
%     \end{center}
%     \caption{ResNet50 on Tiny ImageNet}
%     \label{fig:tiny}
% \end{figure}

% \begin{figure}[H]
%     \begin{center}
%     \includegraphics[width = 1\linewidth]{Figures/CNNonBrainMRI_Cir_White_comparison.png}
%     \end{center}
%     \caption{CNN on Synthetic Brain MRI}
%     \label{fig:synthetic}
% \end{figure}

% \begin{figure}[H]
%     \begin{center}
%     \includegraphics[width = 1\linewidth]{Figures/ResNet50onTrafficLight_comparison.png}
%     \end{center}
%     \caption{ResNet50 on Lisa-Traffic-Light}
%     \label{fig:traffic}
% \end{figure}

\subsection{Evaluation on Execution Time}

Table \ref{execution time} compares the execution time of seven XAI methods across different ResNet models on the ImageNet-S50 dataset. The results highlight significant differences in scalability as model sizes increase. Notably, \ourmethod demonstrates a relatively slow growth in execution time, increasing only marginally from 2.36 seconds for ResNet18 to 2.55 seconds for ResNet101. This indicates that \ourmethod is computationally efficient and well-suited for larger models, maintaining consistent performance with minimal overhead. In contrast, methods like Occlusion and RISE exhibit steep increases in execution time, with Occlusion rising dramatically from 36.38 seconds for ResNet18 to 122.79 seconds for ResNet101, and RISE escalating from 2.82 seconds to 10.34 seconds. SHAP also shows notable growth, doubling its execution time from 2.00 seconds to 4.11 seconds across the same range. 
% While GradSHAP and IntGrad offer faster execution times for smaller models, their computational costs grow significantly for ResNet50 and beyond, indicating limited scalability. h-SHAP consistently maintains the lowest execution time across all model sizes but lacks the comprehensive segmentation and attribution advantages of \ourmethod. 
Overall, \ourmethod achieves a good balance between execution efficiency and explanation quality, making it a robust choice for large-scale applications.

\begin{table}[h]
\caption{Averaged execution time (second per image) of various ResNet models on $224\times224$ images.}
\vskip -0.05in
\label{execution time}
\begin{center}
\resizebox{\columnwidth}{!}{%
\begin{tabular}{cccccccc}
\toprule
     & SHAP & GradSHAP & IntGrad & Occlusion & RISE & h-SHAP & \ourmethod\\ 
\toprule
\toprule
% \bf CNN on Brain Tumor MRI & 1.72 & 0.46 & 0.33 & 16.02 & 0.70 & 0.17 & 2.57\\
ResNet18 & 2.00 & 0.85 & 0.56 & 36.38 & 2.82 & 0.24 & 2.36\\
ResNet34 & 3.28 & 0.84 & 0.58 & 59.90 & 4.32 & 0.24 & 2.38\\
ResNet50 & 3.65 & 1.82 & 1.25 & 88.14 & 5.80 & 0.48 & 2.52\\
ResNet101 & 4.11 & 2.12 & 1.21 & 122.79 & 10.34 & 0.48 & 2.55\\
\bottomrule
\end{tabular}
}
\end{center}
\vskip -0.2in
\end{table}

Table \ref{execution time over layers} compares the execution times of XAI methods as the image size increases, highlighting the scalability differences among the methods. The results reveal that the original SHAP algorithm’s execution time grows exponentially with the number of pixels (features), rising from 1.64 seconds for $32 \times 32$ images to 4.19 seconds for $256 \times 256$ images. This exponential growth reflects SHAP's computational complexity of $\mathcal{O}(2^{|N|})$, where $|N|$ is the number of pixels, underscoring its inefficiency for high-resolution images. In contrast, \ourmethod exhibits a much slower polynomial growth, with execution time increasing from 0.62 seconds for $32 \times 32$ images to 3.19 seconds for $256 \times 256$ images. This scalability advantage arises from its hierarchical structure, which reduces the computational complexity to $\mathcal{O}(2^{\sum_{l=1}^L |\mathcal{G}|^l})$. Specifically, for balanced partitions where $|\mathcal{G}^1| = |\mathcal{G}^2| = \cdots = |\mathcal{G}^L| = n$, the total complexity becomes $\mathcal{O}(|N|^{n\cdot\log_n{2}})$, making \ourmethod far more practical for larger images. GradSHAP and IntGrad demonstrate moderate growth rates but remain less efficient than \ourmethod for larger image sizes. These results demonstrate that \ourmethod effectively balances computational efficiency and explanation quality, particularly for high-resolution image data.

\begin{table}[H]
\caption{Averaged execution time (second per image) of ResNet50 model.}
\vskip -0.05in
\label{execution time over layers}
\begin{center}
\resizebox{\linewidth}{!}{%
\begin{tabular}{cccccccc}
\toprule
   Image size & SHAP & GradSHAP & IntGrad & Occlusion & RISE & h-SHAP & \ourmethod\\ 
\toprule
\toprule
$32 \times 32$ & 1.64 & 0.48 & 0.52 & 1.03 & 1.15 & 0.01 & 0.32\\
$96 \times 96$ & 1.86 & 0.62 & 0.62 & 12.04 & 1.46 & 0.26 & 0.62\\
$128 \times 128$ & 2.12 & 0.88 & 0.69 & 22.35 & 2.56 & 0.28 & 0.91\\
$224 \times 224$ & 3.65 & 1.82 & 1.25 & 88.14 & 5.80 & 0.48 & 2.52\\
$256 \times 256$ & 4.19 & 2.26 & 1.32 & 92.62 & 6.42 & 0.50 & 3.19\\
\bottomrule
\end{tabular}
}
\end{center}
\end{table}

\subsection{Sensitivity Study of Hyperparameters of Segmentation}

In \ourmethod, the initial segmentation consists of two key stages: (1) edge detection via gradient thresholding and (2) edge dilation to form closed feature blocks. The edge detection stage is controlled by two hyperparameters: the lower and upper gradient thresholds ($T_{\text{lower}}$, $T_{\text{upper}}$), while the dilation step is governed by the kernel size. To assess the robustness of our default configuration ($T_{\text{lower}} = 75$th percentile, $T_{\text{upper}} = 90$th percentile, dilation kernel size = $2 \times 2$), we conduct a sensitivity analysis, summarized in Fig.~\ref{fig:Segment_param}. The results indicate that our default hyperparameters yield the most coherent and semantically meaningful segmentations, with the optimal configuration highlighted by arrows in the figure. For example, adjusting the thresholds to $T_{\text{lower}} = 50$ and $T_{\text{upper}} = 90$ (as shown in the second column) leads to imprecise edge detection, compromising feature boundary clarity. Overall, this analysis confirms that the proposed default segmentation settings are empirically well-calibrated for producing reliable and interpretable explanations.

\begin{figure}[h]
    \begin{center}
    \vskip -0.1in
    \includegraphics[width = 0.9\linewidth]{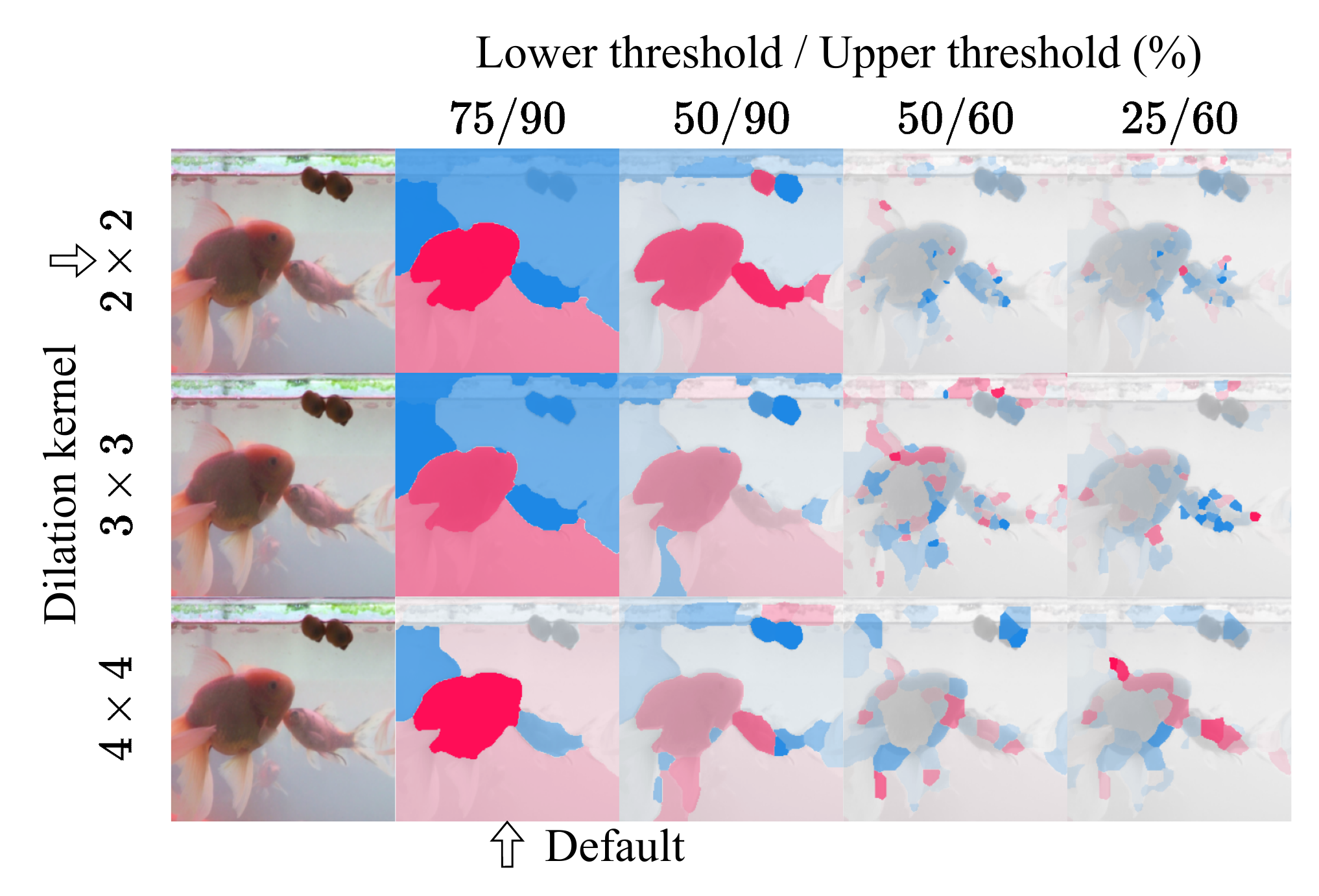}
    \end{center}
    \vskip -0.2in
    \caption{The impact of segmentation parameters on the result.}
    \label{fig:Segment_param}
    \vskip -0.2in
\end{figure}

\subsection{Evaluation on Non-image Dataset}
\label{sec:non-imageDataset}

We also evaluate \ourmethod in non-image dataset, using Adult Census Income dataset~\cite{adult_2}, which is a widely used benchmark in tabular learning. Our objective is to examine whether \ourmethod preserves the intrinsic dependencies among features in its attributions. Specifically, we construct graphs in Fig. \ref{fig:non-image_comp}, where edges represent correlations between features and their attributions. The results show that SHAP’s attribution graph fails to preserve these relationships, diverging from the true feature correlation structure. In contrast, \ourmethod effectively captures feature dependencies, producing an attribution graph that closely aligns with the original correlations. These findings highlight \ourmethod’s ability to enhance both attribution accuracy and interpretability, addressing a critical limitation of SHAP in non-image datasets.

\begin{figure}[h]
    \begin{center}
    \vskip -0.15in
    \includegraphics[width = 1\linewidth]{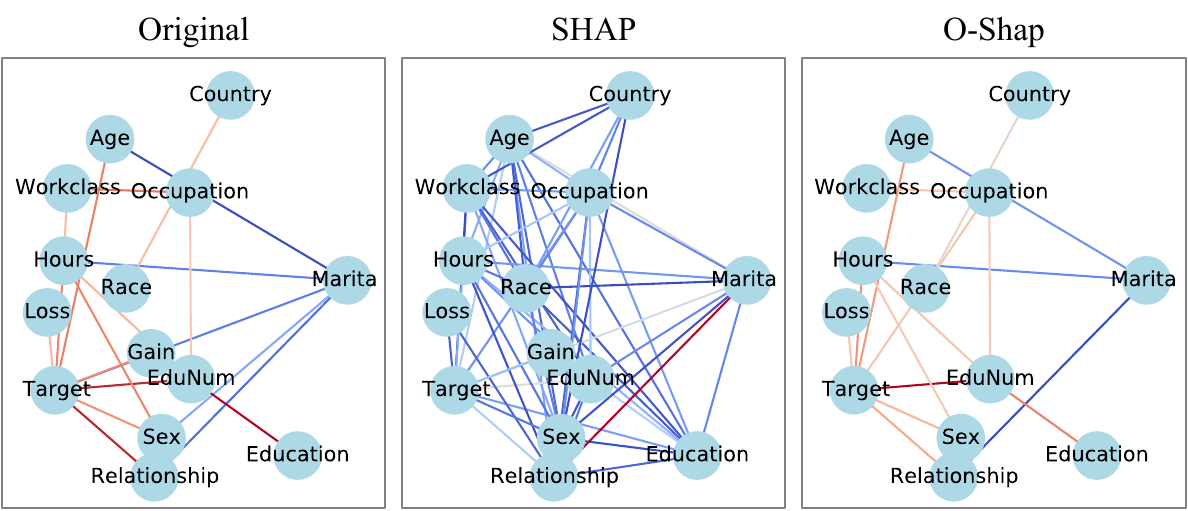}
    \end{center}
    \vskip -0.1in
    \caption{Comparison of correlations between original features, and feature explanation scores of SHAP and \ourmethod.}
    \label{fig:non-image_comp}
\end{figure}

\vspace{-0.5em}
\section{Conclusion}

We present \ourmethod, a model-agnostic, semantics- and hierarchy-aware explanation framework that addresses key limitations of Shapley-based methods in modeling feature dependencies. By leveraging the Owen value and a novel segmentation algorithm satisfying the \emph{positive $\mathcal{T}$-property}, \ourmethod ensures structurally consistent, interpretable, and efficient attributions. Our approach also reduces the exponential cost of SHAP to polynomial time, enabling scalability to high-dimensional domains. Experiments on image and tabular datasets confirm that \ourmethod yields more accurate and context-aware explanations than existing baselines. These results highlight the importance of structure-aware design in advancing trustworthy and scalable explainable AI.

\bibliographystyle{IEEEtran}
\bibliography{reference}

\end{document}